\documentclass{article}

\usepackage{microtype}
\usepackage{graphicx}
\usepackage{subfigure}
\usepackage{booktabs} 

\usepackage{hyperref}

\usepackage{amsmath}
\usepackage{amssymb}
\usepackage{amsthm}
\usepackage{enumitem}
\usepackage{hyperref}
\usepackage{mathtools}
\usepackage{tcolorbox}
\usepackage{xcolor}

\providecommand{\bR}{\mathbb R}

\providecommand{\cH}{\mathcal H}

\DeclareMathOperator*{\argmax}{\arg\!\max}
\DeclareMathOperator*{\argmin}{\arg\!\min}
\DeclareMathOperator{\dist}{{dist}}
\DeclareMathOperator{\relu}{{relu}}

\newcommand{\push}{{\texttt{push}}}
\newcommand{\pop}{\texttt{pop}}

\newcommand{\priority}{\texttt{priority}}
\DeclareMathOperator{\Face}{{Face}}
\DeclareMathOperator{\parent}{{parent}}

\DeclareMathOperator{\class}{{c}}

\newcommand{\decision}{\texttt{decision\_bound}}
\newcommand{\reachable}{\texttt{contains\_db}}
\newcommand{\nextlayer}{\texttt{next\_layer}}
\newcommand{\restrict}{\texttt{restriction}}

\newtheorem{theorem}{Theorem}[section]

\newtheorem{definition}[theorem]{Definition}

\newtheorem{lemma}[theorem]{Lemma}
\newtheorem{proposition}[theorem]{Proposition}

\makeatletter
\newcommand{\dontusepackage}[2][]{%
  \@namedef{ver@#2.sty}{9999/12/31}%
  \@namedef{opt@#2.sty}{#1}}
\makeatother
\dontusepackage{algorithmic}

\usepackage[noend]{algcompatible}

\usepackage{etoolbox}
\usepackage{tikz}
\usetikzlibrary{tikzmark}
\usetikzlibrary{calc}

\errorcontextlines\maxdimen

\newcommand{\ALGtikzmarkcolor}{black}
\newcommand{\ALGtikzmarkextraindent}{4pt}
\newcommand{\ALGtikzmarkverticaloffsetstart}{-.5ex}
\newcommand{\ALGtikzmarkverticaloffsetend}{-.5ex}
\makeatletter
\newcounter{ALG@tikzmark@tempcnta}

\newcommand\ALG@tikzmark@start{%
    \global\let\ALG@tikzmark@last\ALG@tikzmark@starttext%
    \expandafter\edef\csname ALG@tikzmark@\theALG@nested\endcsname{\theALG@tikzmark@tempcnta}%
    \tikzmark{ALG@tikzmark@start@\csname ALG@tikzmark@\theALG@nested\endcsname}%
    \addtocounter{ALG@tikzmark@tempcnta}{1}%
}

\def\ALG@tikzmark@starttext{start}
\newcommand\ALG@tikzmark@end{%
    \ifx\ALG@tikzmark@last\ALG@tikzmark@starttext
    \else
        \tikzmark{ALG@tikzmark@end@\csname ALG@tikzmark@\theALG@nested\endcsname}%
        \tikz[overlay,remember picture] \draw[\ALGtikzmarkcolor] let \p{S}=($(pic cs:ALG@tikzmark@start@\csname ALG@tikzmark@\theALG@nested\endcsname)+(\ALGtikzmarkextraindent,\ALGtikzmarkverticaloffsetstart)$), \p{E}=($(pic cs:ALG@tikzmark@end@\csname ALG@tikzmark@\theALG@nested\endcsname)+(\ALGtikzmarkextraindent,\ALGtikzmarkverticaloffsetend)$) in (\x{S},\y{S})--(\x{S},\y{E});%
    \fi
    \gdef\ALG@tikzmark@last{end}%
}

\apptocmd{\ALG@beginblock}{\ALG@tikzmark@start}{}{\errmessage{failed to patch}}
\pretocmd{\ALG@endblock}{\ALG@tikzmark@end}{}{\errmessage{failed to patch}}
\makeatother

\usepackage[accepted]{icml2020}

\icmltitlerunning{Hierarchical Verification for Adversarial Robustness}

\begin{document}

\twocolumn[
\icmltitle{Hierarchical Verification for Adversarial Robustness}




\begin{icmlauthorlist}
\icmlauthor{Cong Han Lim}{atg}
\icmlauthor{Raquel Urtasun}{atg,tor}
\icmlauthor{Ersin Yumer}{atg}
\end{icmlauthorlist}

\icmlaffiliation{atg}{Uber Advanced Technologies Group, Toronto ON, Canada}
\icmlaffiliation{tor}{Department of Computer Science, University of Toronto, Toronto ON, Canada}

\icmlcorrespondingauthor{Cong Han Lim}{conghan@uber.com}

\icmlkeywords{Machine Learning, ICML}

\vskip 0.3in
]



\printAffiliationsAndNotice{Work done as part of first author's AI Residency Program.}  


\begin{abstract}

We introduce a new framework for the exact point-wise $\ell_p$ robustness verification problem that exploits the layer-wise geometric structure of deep feed-forward networks with rectified linear activations (ReLU networks). 
The activation regions of the network partition the input space, and one can verify the $\ell_p$ robustness around a point by checking all the activation regions within the desired radius. 
The GeoCert algorithm \citep{jordan_provable_2019} treats this partition as a generic polyhedral complex in order to detect which region to check next. 
In contrast, our LayerCert framework considers the \emph{nested hyperplane arrangement} structure induced by the layers of the ReLU network and explores regions in a \emph{hierarchical} manner. 
We show that, under certain conditions on the algorithm parameters, LayerCert \emph{provably} reduces the number and size of the convex programs that one needs to solve compared to GeoCert. 
Furthermore, our LayerCert framework allows  the incorporation of lower bounding routines based on convex relaxations to further improve performance. 
Experimental results demonstrate that LayerCert can significantly reduce both the number of convex programs solved and the running time over the state-of-the-art.
\end{abstract}


\section{Introduction} \label{sec:intro}

Deep neural networks have been demonstrated to be susceptible to adversarial perturbations of the inputs (e.g., \citet{szegedy2013intriguing,biggio2013evasion,goodfellow2014explaining}). 
Hence, it is important to be able to measure how vulnerable a neural network may be to such noise, especially for safety-critical applications.
We study the problem of pointwise \emph{exact} verification for $\ell_p$-norm adversarial robustness for trained deep feed-forward networks with ReLU activation functions. The point-wise $\ell_p$ robustness with respect to an input $x \in \bR^n$ and a classifier $\class:\bR^n \rightarrow [C] \coloneqq \{1,2,3,\dotsc,C\}$ is defined as 
\begin{align}\label{eq:robustness}
\epsilon^*(x;c) \coloneqq \min_{\|v\|_p \leq \epsilon} \epsilon \text{ s.t.\ } \class(x + v) \neq \class(x). 
\end{align}
The goal of  exact or complete robustness verification is to check if $\epsilon > r$ for some desired radius $r$. 
The choices of $p$ studied in the literature are typically $1,2,$ and $\infty$; our work applies to all $p \geq 1$. 
Solving Problem~\eqref{eq:robustness} exactly (or within a factor of $1 - o(1) \ln n$) is known to be NP-hard \citep{weng_towards_2018}. 
Developing methods that perform well in practice would require a better understanding of the mathematical structure of neural networks. 

In this work, we approach the problem from the angle of how to directly exploit the geometry induced in input space by the neural network. 
Each activation pattern (i.e., whether each neuron is on or off) corresponds to a polyhedral region in the input space, and the decision boundary within each region is \emph{linear}.
A natural geometric approach to the verification problem is then to check regions in order of their distance. 
We illustrate this in Figure~\ref{figs:geomethods}.
We can terminate this process either when we have reached the desired verification radius or when we have exceeded the distance to the closest decision boundary found. 
In the latter case the distance to that boundary is the solution to Problem \eqref{eq:robustness}. 

\citet{jordan_provable_2019} proposed the first algorithm for this distance-based exploration of the regions.
Their GeoCert algorithm navigates the regions in order of distance by maintaining a priority queue containing all the polyhedral faces that make up the frontier of all regions that have been visited. 
The priority associated with each face is computed via an optimization problem that can be solved by a generic convex programming solver. 
Under a limited time budget, GeoCert finds a stronger computational lower bound for $\epsilon^*(x;c)$ 
compared to a complete method that directly uses mixed-integer programming \citep{tjeng_evaluating_2019}.

{
\begin{figure}
\begin{center}
\includegraphics[width=0.98\linewidth]{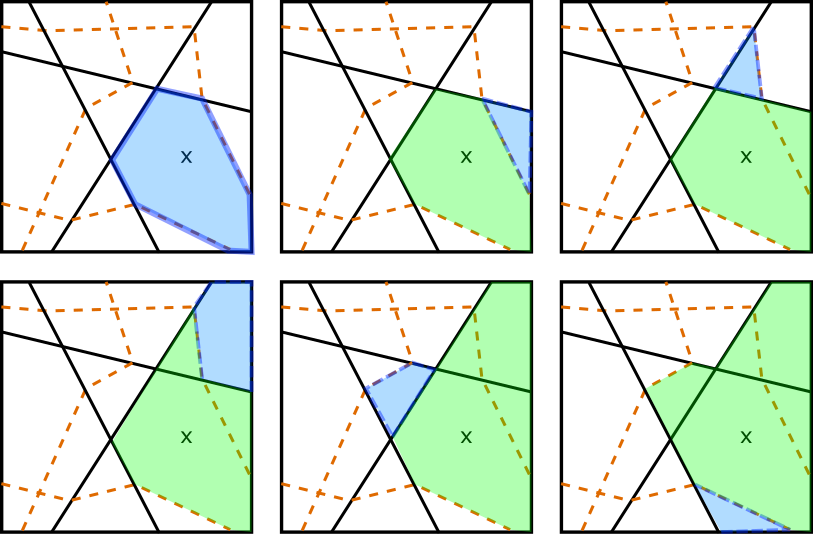}
\vspace{-5pt}
\caption{How geometric methods explore the input space. Each square illustrates one step in the process of exploring regions of increasing $\ell_2$ distance from the initial point $x$.
The box represents the input space to a ReLU network, the inner black lines the regions induced by three first layer ReLUs, and brown lines the regions by another three ReLUs in the second layer. 
The blue regions are being processed during that step, while the green regions have already been processed.\label{figs:geomethods}}
\end{center}
\end{figure}
}

In this paper we introduce the \emph{LayerCert} framework that makes use of the \emph{layer-wise} structure of  neural networks. 
The first layer of ReLUs induce a hyperplane arrangement structure, and each subsequent layer induces one within each region of the hyperplane arrangement from the previous layer. 
This forms a nested hyperplane pattern and a hierarchy of regions/subregions that our algorithm will use to navigate the input space.
This hierarchical approach has two main advantages over GeoCert:
\begin{enumerate}
  \item \emph{Provable} reduction in the number and size of convex programs solved when using a convex distance-based priority function.
  \item The ability to incorporate convex-relaxation-based \emph{lower bounding/incomplete methods} for Problem~\eqref{eq:robustness} to reduce the number of regions that are processed.
\end{enumerate}

We demonstrate the first advantage by studying a simplified version of LayerCert (LayerCert-Basic), which introduces our hierarchical approach to navigating the regions but does not include the use of additional subroutines to prune the search space.
By making use of the nested hyperplane structure, LayerCert \emph{provably} reduces the number of convex programs and the sizes of each program that need to be solved compared to GeoCert when using the same convex distance priority function. 
This is done by identifying a minimal set of programs that are required and by amortizing the work associated with a single region in GeoCert across multiple levels of the hierarchy in LayerCert. 

The second advantage comes from the fact that each region $R$ in the $i$-th level of the hierarchy is associated with a set of children regions in the $(i+1)$-th level that is contained entirely within $R$. 
This allows us to use incomplete verifiers over just this region to determine if the region might intersect with a decision boundary. 
If the verifier returns that no such overlap exists, we can then safely remove the region and all its children from further consideration.
One straightforward way to do this is to use efficient reachability methods such as interval arithmetic \citep{xiang_output_2018}. In this work we also develop a novel method that leverages \emph{linear lower bounds} on the neural network function \citep{weng_towards_2018,zhang_efficient_2018} to construct a half-space that is guaranteed not to contain any part of the decision boundary.
This can restrict the search space significantly.
Furthermore, we can warm start LayerCert by projecting the input point onto the half-space.

Using the experimental setup of \citet{jordan_provable_2019}, we  compare the number of programs solved and overall wall-clock time of different variants of GeoCert and LayerCert. 
Each LayerCert method uses a different combination of lower-bounding methods. 
For GeoCert, we consider different choices of priority functions.
In addition to the standard $\ell_p$ distance priority function, \citet{jordan_provable_2019} describes a non-convex priority function that incorporates a Lipschitz term. 
This Lipschitz variant modifies the order in which regions are processed and also provides an alternative warm-start procedure. 
Our LayerCert variants consistently outperform GeoCert using just the $\ell_p$ distance priority function and in most experiments   our lower-bounding techniques outperform the Lipschitz variant of GeoCert. 




\paragraph{Notation.}
We use $[k]$ to denote the index set $\{1,2,\dotsc,k\}$.
Superscripts are used to index distinct objects, with the exception of $\bR^n$ and $\bR^+$ to denote the $n$-dimensional Euclidean space and the nonnegative real numbers respectively. 
We use subscripts for vectors and matrices to refer to entries in the objects and subscripts for sets to refer to distinct sets. We use \texttt{typewriter\_fonts} to denote subroutines. 



\section{Related work} \label{sec:related}
Besides GeoCert \citep{jordan_provable_2019}, the majority of exact or complete verification methods are based on branch-and-bound (e.g. \citet{katz_reluplex:_2017,wang_formal_2018,tjeng_evaluating_2019,anderson_optimization_2019,lu_neural_2020}), and \citet{bunel_unified_2018,bunel_branch_2019} provide a detailed overview. 
These methods construct a search tree over the possible individual ReLU activations and use upper and lower bounds to prune the tree. The upper bounds come from adversarial examples, while the lower bounds are obtained by solving a relaxation of the original problem, which we briefly discuss in the next paragraph. 
Since our focus in this work is on methods that directly leverage the geometry of the neural network function in the input space, we leave a detailed comparison against these methods to future work.

Instead of exactly measuring or verifying, we can instead overapproximate or relax the set reachable by an $\epsilon$-ball around the input point (e.g. \citet{dvijotham_dual_2018,xiang_output_2018,singh_fast_2018,weng_towards_2018,wong_provable_2018,zhang_efficient_2018,singh_abstract_2019}). 
These \emph{incomplete} approaches give us a \emph{lower bound} for Problem~\eqref{eq:robustness}. 
These can only certify that there is no perturbation that changes the class within some radius $r$ for some $r$ that can be much smaller than $\epsilon^*(x;c)$ 
The majority of the works focus on different convex relaxations of the ReLU activations (see \citet{salman_convex_2019} for a discussion of the tightness of the different approaches), though recent works have started going beyond single ReLU relaxations \citep{singh_beyond_2019,anderson_strong_2019}.
Another line of work studies how to efficiently bound the Lipschitz constant of the neural network function \citep{szegedy2013intriguing,bartlett2017spectrally,balan2017lipschitz,weng_towards_2018,combettes2019lipschitz,fazlyab_efficient_2019,zou2019lipschitz}.

The complexity of geometric methods (and many other exact methods) for robustness verification can be upper bounded by a constant $\times$ (number of activation regions) $\times$ (complexity of solving a convex program). 
There have been several works in recent years that study the number of these activation regions in ReLU networks. 
The current tightest upper and lower bounds for the maximum number of nonempty activation regions, which are exponential in the number of ReLU neurons, is given by \citet{serra_bounding_2018}. 
\citet{hanin_deep_2019} provide an upper bound on the expected number of regions (exponential in the lesser of the input dimension and the number of ReLU neurons).


\section{Hierarchical Structure of ReLU Networks} \label{sec:structure}

In this section, we describe the structure induced by the ReLU neurons in the input space and how this leads naturally to a hierarchy of regions that we use in our approach.

\subsection{Hyperplane Arrangements}

\begin{definition} (Hyperplanes and hyperplane arrangements)
A hyperplane $H \subset \bR^n$ is an $(n-1)$-dimensional affine space that can be written as $\{x \:|\: a^\intercal x = b \}$ for some $a \in \bR^n$ and $b \in \bR$. A hyperplane arrangement $\cH$ is a finite set of hyperplanes.
\end{definition}

\begin{definition} (Halfspaces)
Given a hyperplane $H \coloneqq \{x \:|\: a^\intercal x = b \}$, the halfspaces $H^\leq$ and $H^\geq$ correspond to the sets $\{x \:|\: a^\intercal x \leq b \}$ and $\{x \:|\: a^\intercal x \geq b \}$, respectively. 
A polyhedron is a finite intersection of halfspaces.
\end{definition}

\begin{definition} (Patterns and regions in hyperplane arrangements)
Given a hyperplane arrangement $\cH$ and a pattern $P: \cH \rightarrow \{ -1,1 \}$, the corresponding region is $$R_+ \coloneqq \bigcap_{P(H) = -1} H^{\leq} \cap \bigcap_{P(H) = 1} H^{\geq}.$$ We say that another pattern $Q: \cH \rightarrow \{ -1,1 \}$ (or region $R_Q$) is a neighbor of $P$ ($R_P$ resp.) if $P$ and $Q$ differ only on a single hyperplane.
\end{definition}


\subsection{Geometric Structure of Deep ReLU Networks} 
\label{subsec:deep_geo}

A 
ReLU network with $L$ hidden layers for classification with $C$ classes and $n_i$ neurons in the $i$-th layer for $i \in [L]$ can be represented as follows:
\begin{align}
  \notag x^0 &\coloneqq x & &\text{input},\\
  z^{i+1} &\coloneqq W^i x^i + b^i & &\text{for } i \in \{0, 1,\dotsc, L-1 \}, \label{eq:z} \\
  \notag x^{i} &\coloneqq \relu(z^i) & &\text{for } i \in \{1,\dotsc, L \}, \\
  \notag f(x) &\coloneqq W^L x^L + b^L & &\text{output}, 
\end{align}
where $W^i,b^i$ describe the weight matrix and bias in the $i$-th layer and the ReLU function is defined as $\relu(v))_i \coloneqq \max(0,v_i)$. 
The length of $b^i$ is $n^{i+1}$ for $i \in [L-1]$ and $b^0 \in \bR^n$ and $b^L \in \bR^C$. 
The classification decision is given by $\argmax_i (f(x))_i$, and the decision boundary between classes $i$ and $j$ is the set $\{ x \:|\: f_i(x) - f_j(x) = 0 \}$.

Note that layers like batch normalization layers, convolutional layers, and average pooling layers can be included in this framework by using the appropriate weights and biases. 
We can also have a final softmax layer since it does not affect the decision boundaries as it is a symmetric monotonically increasing function. 

\begin{definition} \label{def:neighbors} (Full activation patterns)
Let $n_i$ denote the number of neurons in the $i$-th layer.
A (full) activation pattern $A = (A_1,\dotsc,A_L)$ is a collection of functions $A_i: [n_i] \rightarrow \{-1,+1\}$. Two activation patterns are neighbors if they differ on exactly one layer for exactly one neuron.
\end{definition}

We can define a \emph{neighborhood graph} over the set of activation patterns where each node presents a pattern and we add an edge between neighboring patterns.

\begin{definition} \label{def:full_activation} (Activation patterns and the input space)
For an input $x$, the (full) activation pattern of $x$ is $A^x$ where
\begin{align}
A^x_i(j) \coloneqq 
\begin{cases}
+1 &\text{if } z^i_j \geq 0, \\
-1 &\text{if } z^i_j < 0,
\end{cases}
\end{align}
where the $z^i_j$ terms are defined according to \eqref{eq:z}. Conversely, given an activation pattern $A$, the corresponding \emph{activation region} is
$R_A \coloneqq \{ x \:|\: A^x = A \}.$
\end{definition}

Given an activation pattern $A$ and some $x \in R_A$, the corresponding $z$ terms are given by
\begin{align} \label{eq:z2}
z^{i+1} = W^i I^{A_i} z^i + b^i
\end{align} 
where $I^{A_i}$ is the diagonal 0-1 matrix such that
\begin{align*}
I^{A_i}_{j,k} 
\coloneqq 
\begin{cases}
1 &\text{if } j = k \text{ and } A_i(j) = 1, \\
0 &\text{otherwise.}
\end{cases}
\end{align*}
By letting 
\begin{align*}
c^i \coloneqq \sum_{j=0}^i \left( \left( \prod_{k=j+1}^i W^k I^{A_k} \right) b^j \right)
\end{align*}
we can expand  
Eq. \eqref{eq:z2} as
\begin{align} \label{eq:z3}
z^{i+1} = \left(\prod_{j=0}^i W^j I^{A_j}\right) x + c^i.
\end{align} 
Hence, each $z^i$ term and $f(x)$ can be expressed as a linear expression over $x$ involving $W^i$, $I^{A_i}$, and $b^i$ terms.
This also allows us to write $R_A$ in the form of linear inequalities over $x$ as a polyhedron 
\begin{align} \label{eq:region_poly}
\{ x \:|\: A_i(j) z^i_j \geq 0 \text{ for } i \in [L], j \in [n_i] \}
\end{align}



Since the neural network function $f$ is linear within each activation region, each decision boundary is also linear within each activation region. This allows us to efficiently compute the classification decision boundaries.

The set of activation regions for a network with one hidden layer corresponds to the regions of a hyperplane arrangement (where each row of $W^0$ and the corresponding entry in $b^0$ defines a hyperplane). 
With each additional layer, we take the regions corresponding to the previous layer and add a hyperplane arrangement to each region. 
Thus, this leads to a \emph{nested hyperplane arrangement}. Figure~\ref{fig:nested} illustrates this structure and the corresponding neighborhood graph. 

\begin{figure}
\centering
\includegraphics[width=0.98\linewidth]{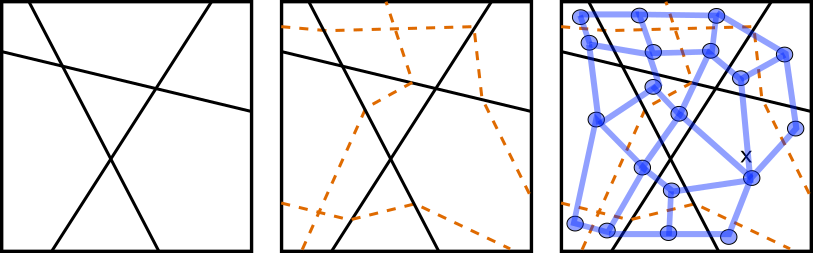}
\caption{Left: the activation regions of a ReLU network with one hidden layer forms a hyperplane arrangement. Middle: the activation regions of a ReLU network with two hidden layers. Note that within the regions defined by the previous layer, the lines induce a hyperplane arrangement.
Right: The neighborhood graph of the regions.
\label{fig:nested}} 
\end{figure} 

In addition to full activation patterns, it is useful to consider the patterns for all neurons up to a particular layer.

\begin{definition} (Partial activation patterns and regions) 
Given some $l < L$, an $l$-layer partial activation pattern $A = (A_1,\dotsc,A_l)$ is a collection of functions $A_i: [n_i] \rightarrow \{-1,+1\}$. 
The corresponding partial activation region $R_A$ is $\{ x \:|\: A_i(j) z^i_j \geq 0 \text{ for } i \in [L], j \in [n_i] \}$.
\end{definition}

The partial activation regions naturally induce a hierarchy of regions. We can describe the relationship between the regions in the different levels in the following terms:

\begin{definition} \label{def:parent}
For a $l$-layer activation pattern $A=(A_1,\dotsc,A_l)$, let $\parent(A) \coloneqq (A_1,\dotsc,A_{l-1})$.
The terms \emph{child} and \emph{descendant} are defined analogously. 
\end{definition}

\begin{definition} \label{def:partial_neighbors}
Two $l$-layer partial activation patterns are siblings if they share the same parent pattern.
They are neighboring siblings if they differ on exactly one neuron in the $l$-th layer and agree everywhere else.
\end{definition}

We can use Defintion~\ref{def:parent} and \ref{def:partial_neighbors} to define a \emph{hierarchical search graph} with $L+1$ levels. 
The nodes in the $l$-th level represent the $l$-layer activation patterns. 
We connect two activation patterns in the same level if they are neighboring siblings.
We connect a pattern $A$ to its parent $\parent(A)$ if $R_A$ is the region closest to the input point $x$ out of all its siblings.
We introduce a single node in the level $0$ and connect it to the first-layer activation pattern that contains $x$.
Figure~\ref{fig:hier} illustrates this hierarchy of activation regions and the corresponding hierarchical search graph. 
In Sections~\ref{sec:geo} and \ref{sec:full}, we describe how to leverage this hierarchical structure to design efficient algorithms for verification.

\begin{figure}
\centering
\includegraphics[width=0.8\linewidth]{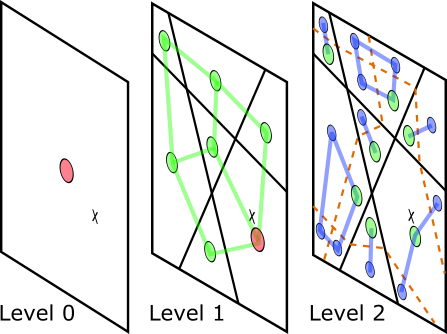}
\caption{The hierarchical structure induced by the partial activation regions. Each partial region is marked by a circle node. The nodes of the same color represent the closest subregion within a parent region to the input point $x$. For example, the upper left most green nodes in levels 1 and 2 are connected. \label{fig:hier}} 
\end{figure} 


\section{Exploring Activation Regions Geometrically} \label{sec:geo}

In this section, we first describe the GeoCert algorithm \citep{jordan_provable_2019} that provides a method for navigating the activation regions in order of increasing distance from the input. 
We subsequently consider the hierarchy of partial regions and describe LayerCert-Basic that leverages the geometric structure to provably reduce the number and size of convex programs compared to GeoCert. 
We then introduce the full LayerCert framework in Section~\ref{sec:full}.

\subsection{Prior Work: GeoCert}

GeoCert performs a process akin to breadth-first search over the neighbourhood graph (see Figure~\ref{fig:nested}). 
In each iteration, GeoCert selects the activation region corresponding to the closest unexplored node neighbouring an explored node and then computes the distance to all regions neighbouring  the selected region.

We provide a formal description of GeoCert in Algorithm~\ref{alg:geocert} and demonstrate an iteration of the algorithm in Figure~\ref{fig:geocert}.
Setting the input $U$ term to a radius $r$ solves the robustness verification problem for that radius, while setting it sufficiently high measures the robustness (i.e., Problem~\eqref{eq:robustness}).

We describe the subroutines in detail below.

\begin{algorithm}[t]
   \caption{GeoCert}
   \label{alg:geocert}
\begin{algorithmic}[1]
  \STATE {\bfseries Input:} $x$, $y$ (label of $x$), $U$ (upper bound)
  \STATE $A^x \gets $ activation pattern of $x$
  \STATE $Q \gets $ empty priority queue
  \STATE $Q.\push((0,A^x))$
  \STATE $S \gets \emptyset$
  \WHILE{$Q \neq \emptyset$}
    \STATE $(d, A) \gets Q.\pop()$
    \IF{$A \in S$}
      \STATE{\textbf{continue}}
    \ENDIF
    \STATE $S \gets S \cup \{ A' \}$
    \IF{$U \leq d$}
      \STATE \textbf{return} $U$
    \ENDIF
    \STATE $U \gets \min(\decision(A, x, y), U)$
    \FOR{$A' \in N(A)\setminus S$}
      \IF{$\Face(A,A')$ is nonempty}
        \STATE $d' \gets \priority(x, \Face(A,A'))$ 
        \STATE $Q.\push((d',A'))$
      \ENDIF
    \ENDFOR
  \ENDWHILE
\end{algorithmic}
\end{algorithm}

\paragraph{Measuring the distance to a distance boundary restricted to a region.}
The subroutine $\decision$ with inputs $A,x,y$ solves the problem
\begin{align} \label{eq:decision}
\min_{v} \| x - v \|_p \quad\text{ s.t.\ }\quad v \in R_A,f_y(v) - f_j(v) = 0
\end{align}
for all classes $j \neq y$ and returns the minimum (or $\infty$ if all the problems are infeasible). 
This is equivalent to computing the $\ell_p$ projection of $x$ onto the respective set. We can use algorithms specifically designed for projection onto sets such as Dykstra's method \citep{boyle_method_1986}. We can also use generic convex optimization solvers to handle a wider range of priority functions. 

\paragraph{Computing the priority function.}
From \eqref{eq:region_poly}, we can write each activation region $R_A$ as a polyhedron $\{ x \:|\: A_i(j) z^i_j \geq 0 \text{ for } i \in [L], j \in [n_i] \}$. 
A neighbouring region $R_A'$ that differs only on the $a$-th neuron on the $b$-th layer will intersect with $R_A$ within the hyperplane $H = \{ x \:|\: z^b_a \geq 0\}$. We name the set $R_A \cap H = R_A' \cap H$ as $\Face(A,A')$ since this set is a face of both $A$ and $A'$. As with $R_A$ and \eqref{eq:region_poly}, we can write $\Face(A,A')$ as
\begin{align} \label{eq:rface_poly}
\{ x \:|\: A_i(j) z^i_j \geq 0 \text{ for } i \in [L], j \in [n_i], z^a_b = 0  \}
\end{align}
where $z^i_j$ can be expressed in terms of the $x$ variables using the expression in \eqref{eq:z3}.
Given some function $q:\bR^n \times \bR^n \rightarrow \bR^+$, the subroutine $\priority$ with inputs $x,\Face(A,A')$ solves the following optimization problem:
\begin{align} \label{eq:priority}
\min_{v} q(x,v) \quad\text{ s.t.\ }\quad v\in \Face(A,A').
\end{align}
For $\ell_p$-robustness, a natural choice of $q$ is the $\ell_p$ distance function. \citet{jordan_provable_2019} also propose an alternative variant of GeoCert where they incorporate an additional term $\min_{y\neq j}\frac{f_y(v) - f_j(v)}{L}$ priority function, where $L$ denotes an upper bound on the Lipschitz constant. This makes the priority function nonconvex. We will refer to the $\ell_p$ priority variant of GeoCert as just GeoCert and specifically use the term GeoCert-Lip for the Lipschitz variant.

\begin{figure}[t]
\centering
\includegraphics[width=0.98\linewidth]{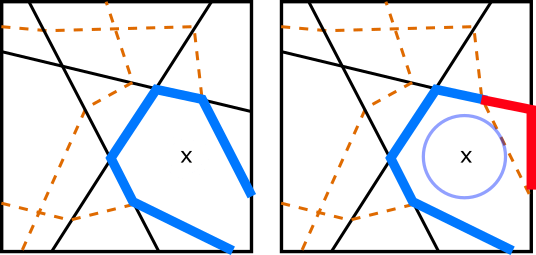}
\caption{One iteration of GeoCert \label{fig:geocert}. The ball represents a level set of the priority function and the blue line the set of boundary faces of the explored regions. After popping the nearest unexplored activation region off a priority queue, GeoCert computes the distance of previously unseen faces in that region (marked in red).} 
\end{figure}


\subsection{Our Approach -- LayerCert} \label{subsec:layercert-basic}

Instead of exploring the neighborhood graph of full activation regions, we develop an algorithm that makes use of the nested hyperplane arrangement and the graph induced by it. See Definition~\ref{def:parent} and Figure~\ref{fig:hier} for a description of this graph.
We first give a description of a basic form of LayerCert in Algorithm~\ref{alg:layercert-basic}, followed by theoretical results about this method.


In each iteration, LayerCert processes the next nearest \emph{partial activation pattern} by computing the distances to all \emph{sibling patterns} to the queue.
This is in contrast to GeoCert which computes the distance to all neighbouring patterns of the next nearest \emph{full} activation pattern.


The new subroutine in LayerCert-Basic is $\nextlayer$, which takes an $l$-layer activation pattern $A$ and an input vector $v$ and returns the $l+1$ layer activation pattern that is a child of $A$ and contains $v$:
\begin{align*}
\nextlayer(A,v) 
\coloneqq &(A_1,\dotsc,A_l,A^v_{l+1})
\end{align*}
where $A^v$ is the full activation pattern of $v$ (Definition~\ref{def:full_activation}).


We prove in the appendix that when the priority function is a convex distance function, LayerCert-Basic visits full activation patterns in ascending order of distance, which implies that it returns the correct solution.

\begin{algorithm}[t]
   \caption{LayerCert-Basic}
   \label{alg:layercert-basic}
\begin{algorithmic}[1]
  \STATE {\bfseries Input:} $x$, $y$ (label of $x$), $U$ (upper bound)
  \STATE $A^x \gets $ activation pattern of $x$ at first layer
  \STATE $Q \gets $ empty priority queue
  \STATE $Q.\push((0,A^x,x))$
  \STATE $S \gets \emptyset$
  \WHILE{$Q \neq \emptyset$}
    \STATE $(d, A, v) \gets Q.\pop()$
    \IF{$U \leq d$}
      \STATE \textbf{return} $U$
    \ENDIF
    \IF{$A$ is a full activation pattern}
      \STATE $U \gets \min(\decision(A, x, y),U)$
    \ELSE
      \STATE $Q.\push(d,\nextlayer(A,v),v)$
    \ENDIF
    \FOR{$A' \in N_{\text{current\_layer}}(A)\setminus S$}
      \IF{$\Face(A,A')$ is nonempty}
        \STATE $v',d' \gets \priority(x, \Face(A,A'))$ 
        \STATE $Q.\push((d',A',v'))$
        \STATE $S \gets S \cup \{ A' \}$
      \ENDIF
    \ENDFOR
  \ENDWHILE
\end{algorithmic}
\end{algorithm}

\begin{theorem}\label{thm:correctness} (Correctness of LayerCert-Basic)
If the priority function used is a convex distance function, LayerCert processes full activation patterns in ascending order of distance and returns the distance of the closest point with a different class from $x$.
\end{theorem}

As a corollary of Theorem~\ref{thm:correctness}, LayerCert-Basic and GeoCert equipped with the same priority function visit the activation patterns in the same order (allowing for permutations of patterns with the exact same priority). 

The primary difference in computational difficulty between the two methods is that the number and complexity of $\priority$ computations is reduced in LayerCert. There are three main reasons for this.
First, LayerCert-Basic updates the set of seen patterns $S$ with $A$ once we have processed any neighbour of $A$. 
Hence, we only do a single $\priority$ computation for each $A$. 
This is not necessarily the case in GeoCert. 
Secondly, the convex programs corresponding to nodes further up the hierarchy have less constraints since they only need to consider all neurons to the respective layer. 
Finally, in LayerCert-Basic we do not need to compute $\priority$ between two $l$-layer partial activation patterns that differ on exactly a single neuron that is not in layer $l$ (i.e.\ these patterns are not siblings). 
GeoCert in contrast computes the priority between any two neighbours as long as the corresponding $\Face$ is non-empty. 
This allows us to amortize the number of convex programs to compute for a single full activation region in GeoCert over multiple levels in LayerCert ---
Consider a full activation pattern $A = (A_1,\dotsc,A_L)$ and the set of all ancestor patterns $B^1,\dotsc,B^{L-1}$ where $B^i \coloneqq (A_1,\dotsc,A_i)$. 
We have $\sum_i^L n_i$ neighbours for GeoCert when processing pattern $A$. 
For LayerCert-Basic, we have up to $n_i$ siblings when processing each $B^i$, for a total of $\sum_i^L n_i$ patterns when processing $B^1,\dotsc,B^{L-1},A$.

These facts can be used to prove our main theoretical result about the complexity of LayerCert-Basic. 
The proof of these and the main theorem are in the supplementary material.


\begin{theorem}\label{thm:main} (Complexity of LayerCert-Basic)
Given an input $x$, suppose the distance to the nearest adversary  returned by LayerCert/GeoCert is not equal to the distance of any activation region from $x$.
Suppose we formulate the convex problems associated with $\decision$ and $\nextlayer$ using Formulations~\eqref{eq:region_poly} and \eqref{eq:rface_poly}. 
We can construct an injective mapping from the set of convex programs solved by LayerCert to the corresponding set in GeoCert such that the constraints in the LayerCert program is a subset of those in the corresponding GeoCert program.
\end{theorem}


\section{The LayerCert Framework} \label{sec:full}

\begin{algorithm}[t]
   \caption{LayerCert Framework}
   \label{alg:layercert}
\begin{algorithmic}[1]
  \STATE {\bfseries Input:} $x$, $y$ (label of $x$), $U$ (upper bound)
  \STATE $d,A,v,M \gets \restrict(x,U)$
  \STATE $Q \gets $ empty priority queue
  \STATE $Q.\push((d,A,v))$
  \STATE $S \gets \emptyset$
  \WHILE{$Q \neq \emptyset$}
    \STATE $(d, A, v, M) \gets Q.\pop()$
    \IF{$U \leq d$}
      \STATE \textbf{return} $U$
    \ENDIF
    \IF{$A$ is a full activation pattern}
      \STATE $U \gets \min(\decision(A, x),U)$
    \ELSE
      \IF{$\reachable(A,x,U) =$ `maybe'}
          \STATE $(d', A', v) \gets \nextlayer(A,v)$ 
          \STATE $Q.\push(d',A',v,M)$
      \ENDIF
    \ENDIF
    \FOR{$A' \in N_{\text{current\_layer}}(A)\setminus S$}
      \IF{$\Face(A,A') \cap M$ is nonempty}
        \STATE $v',d' \gets \priority(x, \Face(A,A') \cap M)$ 
        \STATE $Q.\push((d',A',v',M))$
        \STATE $S \gets S \cup \{ A' \}$
      \ENDIF
    \ENDFOR
  \ENDWHILE
\end{algorithmic}
\end{algorithm}

We now describe how our hierarchical approach is amenable to the use of convex relaxation-based lower bounds  to prune the search space. 
For simplicity, in this section we will assume that we are performing verification with respect to a \emph{targeted} class $y'$.
The general LayerCert framework is presented in Algorithm \ref{alg:layercert}.
The two new subroutines in the algorithm are $\reachable$ and $\restrict$.

Let $B_{p,r}(x)$ denote the $\ell_p$-ball of radius $r$ around $x$.
The subroutine $\reachable(A,x,r)$ returns `false' when $R_A \cap B_{p,r}(x)$  is \emph{guaranteed not} to intersect with a decision boundary. This means that we can remove $A$ and all its descendants from consideration. 
Otherwise, it returns `maybe' and we proceed on to the next level.

The routine $\restrict(x,r)$ explicitly computes a convex set $M$ that is a superset of the part of the decision boundary contained in $B_{p,r}(x)$. 
This both restricts the search directions we need to consider and also allows us to \emph{warm start} the algorithm by projecting the initial point onto this set.
In the following we describe some possible choices for these subroutines. 
In the appendix we discuss a version of LayerCert that recursively applies a modified version of $\restrict$ to aggressively prune the search space. 

\paragraph{Pruning partial activation regions.}

We can use incomplete verifiers (see Section~\ref{sec:related} for references) to check if the region $B_{p,U}$ \emph{might} contain a decision boundary. 
These methods work by implicitly or explicitly computing some overapproximation of the set 
\begin{align} \label{eq:overapprox}
\{ f_y(v) - f_{y'}(v) \:|\: v \in B_{p,U}  \}.
\end{align}
If all values in \eqref{eq:overapprox} are strictly positive, then we know that the decision boundary cannot be in $R_A \cap B_{p,U}$.

Many incomplete verifiers work by constructing convex relaxations of each nonlinear ReLU function. 
If a neuron is guaranteed to be always on or always off over all points in the region of interest, we can use this to tighten the convex relaxation.
This allows us to incorporate information from the current partial activation region into incomplete verifiers.
For our experiments, we choose to use the efficient interval arithmetic approach (first applied in the context of neural network verification by \citet{xiang_output_2018}). 
Below we describe a version of the method that incorporates information about partial regions:
\begin{align*}
  \notag X^0 &\coloneqq B_{p,U}(x),\\
  Y^i &\coloneqq \text{Interval}_A(X^{i-1}) & &\text{for } i \in [L+1],   \\
  Z^{i} &\coloneqq W^{i-1} Y^i + b^{i-1} & &\text{for } i \in [L], \\
  X^i &\coloneqq \relu(Z^i) & & \text{for } i \in [L],
\end{align*}
where $\text{Interval}_A(X)$ is given by the following set: 
$$\left\{ v \:\middle|\: \min_{w \in X} w_j \leq v_j \leq \min(p_j,\max_{w \in X} w_j) \right\}$$ where $p_j$ is $0$ if $A(i,j) = -1$ and $\infty$ otherwise. 
We can use the fact that $f(B_{p,U}(x) \cap R_A) \subseteq W^L Y^{L+1} + b^L$ to check if we need to explore the descendants of $A$.

\paragraph{Warm starts via restricting search area.}

Certain incomplete verifiers implicitly generate enough information to allow us to efficiently compute a convex set $M$ that contains the decision boundary. 
We will describe how to do this for verifiers based on simple linear underestimates of $f_y - f_{y'}$ in the ball $B_{p,U}(x)$ such as  Fast-Lin \citep{weng_towards_2018} and CROWN \citep{zhang_efficient_2018}. 
These methods construct a linear function $g \leq f_y - f_{y'}$ by propagating linear and upper bounds of the ReLUs through the layers. The resulting set $\{v \:|\: g(v) \leq 0\}$ is a halfspace that we can efficiently project onto.
Figure \ref{fig:warmstart} illustrates this concept.
\begin{figure}
\centering
\includegraphics[width=0.98\linewidth]{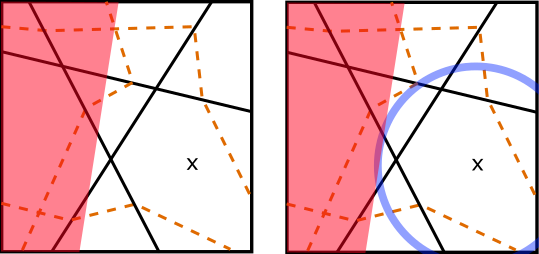}
\caption{Computing a set that contains the decision boundary and warm starting. 
The red shaded region corresponds to the region that contains any potential decision boundary. 
By identifying this region, we can warm start the algorithm by initializing the search at the closest point in the red region. 
\label{fig:warmstart}} 
\end{figure}
For the purposes of this paper we use the lower bound from the CROWN verifier \citep{zhang_efficient_2018} to compute a halfspace $M$.

Once we have computed $M$, we can use it throughout our algorithm. In particular, if a face does not overlap with $M$, we can remove it from consideration in our algorithm. We discuss this in more detail in the supplementary.

Instead of just using linear approximations, we can also use tighter approximations such as the linear programming relaxation that models single ReLUs tightly (see for example \citet{salman_convex_2019})
These result in a smaller convex set $M$ that is still guaranteed to contain the decision boundary but overlaps with less regions, resulting in a reduction in the search space.
The drawback of using such methods is that the representation of $M$ can get significantly more complicated, which in turn increases the cost of solving Problems~\eqref{eq:priority} and \eqref{eq:decision}.

{
\begin{table*}
\begin{center}
\includegraphics[width=1.0\linewidth]{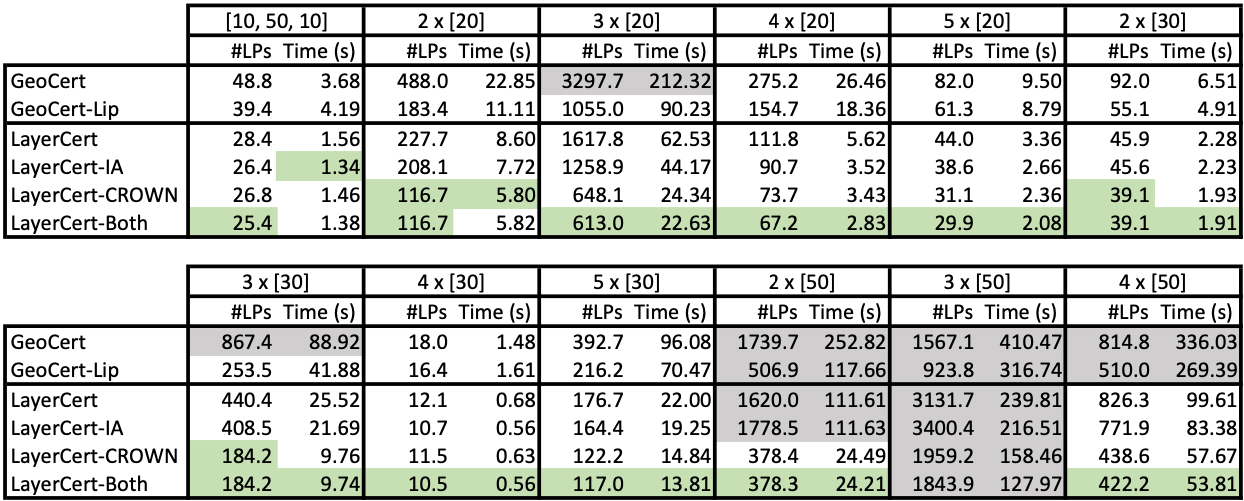}
\vspace{-5pt}
\caption{
Average number of convex programs and running time over 100 inputs for 12 different networks for $\ell_\infty$-distance.
The green shaded entries indicate the best performing method for each neural network under each metric.
The gray shaded entries indicates the algorithm timed out before an exact solution to Problem~(\ref{eq:robustness}) could be found for \emph{at least one input}.
Whenever a timeout occurs, we use a time of 1800 seconds in its place, which leads to an underestimate of the true time.
\label{tb:results}}
\end{center}
\end{table*}
}

\section{Experimental Evaluation} \label{sec:experiments}


We use an experimental setup similar to the one used in \citet{jordan_provable_2019}. 
The two experiments presented here are taken directly from them, except for an additional neural network in the first and a larger neural network in the second.
Additional results are presented in the appendix. \

We consider a superset of the networks used in GeoCert. 
The fully connected ReLU networks we use in the paper have two to five hidden layers, with between 10 to 50 neurons in each layer.
As with \citet{jordan_provable_2019}, we train our networks to classify the digits 1 and 7 in the MNIST dataset and used the same training parameters.
We consider $\ell_\infty$ distance here and also $\ell_2$ distance in the appendix. 

\paragraph{Methods evaluated.}
We test the following variants of GeoCert and LayerCert. 
All variants of LayerCert use an $\ell_p$ distance priority function.
\begin{itemize}
  \item \vspace{-3pt} GeoCert with $\ell_p$ distance priority function.
  \item \vspace{-3pt} GeoCert-Lip: GeoCert with $\| x-v \|_p + \min_{y\neq j}\frac{f_y(v) - f_j(v)}{L}$ priority function, where $L$ denotes an upper bound on the Lipschitz constant found by the Fast-Lip method \citep{weng_towards_2018}.
  \item \vspace{-3pt} LayerCert-Basic.
  \item \vspace{-3pt} LayerCert with interval arithmetic pruning.
  \item \vspace{-3pt} LayerCert with warmstart in the initial iteration using CROWN \citep{zhang_efficient_2018} to underestimate $f_y - f_j$.
  \item LayerCert-Both: with both interval arithmetic pruning and CROWN-based warm start.
\end{itemize}
We initialize each algorithm with a target verification radius of $0.3$.
Each method terminates when we have exactly computed the answer to \eqref{eq:robustness} 
 or when it has determined that the lower bound is at least the radius.
 
 {
\begin{figure*}[t]
\begin{center}
\includegraphics[width=0.2\linewidth]{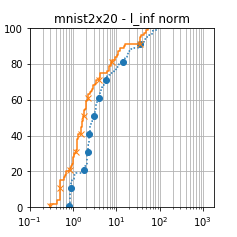}
\includegraphics[width=0.2\linewidth]{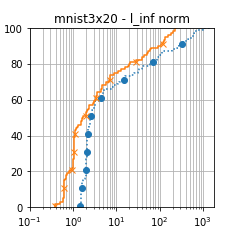}
\includegraphics[width=0.2\linewidth]{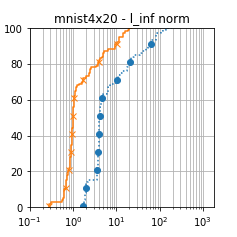}
\includegraphics[width=0.2\linewidth]{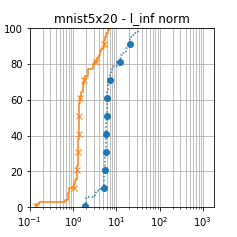} \\
\vspace{-5pt}
\includegraphics[width=0.2\linewidth]{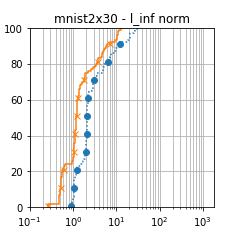}
\includegraphics[width=0.2\linewidth]{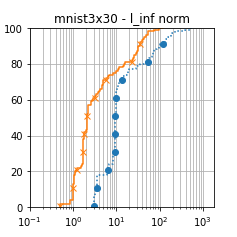}
\includegraphics[width=0.2\linewidth]{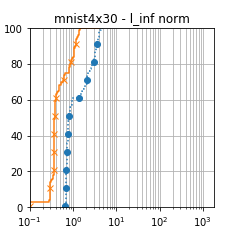}
\includegraphics[width=0.2\linewidth]{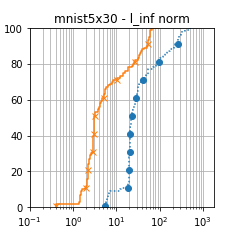} \\
\vspace{-5pt}
\includegraphics[width=0.2\linewidth]{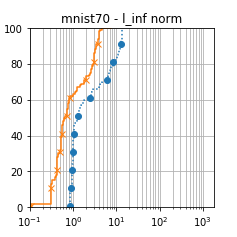}
\includegraphics[width=0.2\linewidth]{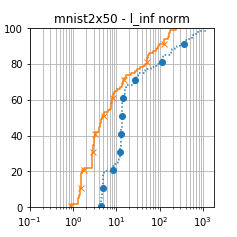}
\includegraphics[width=0.2\linewidth]{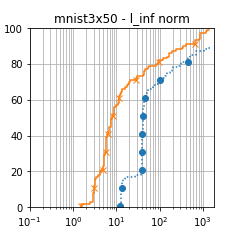}
\includegraphics[width=0.2\linewidth]{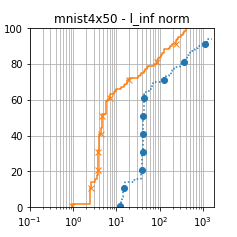}
\vspace{-5pt}
\caption{Performance profiles comparing LayerCert-Both (orange, `x' markers) with GeoCert-Lip (blue, dotted, circle markers) over 12 different networks using $\ell_\infty$ norm. Marks are placed for every 10 input points.
\label{fig:perfprof_inf}}
\end{center}
\end{figure*}
\vspace{-10pt}
}

\paragraph{Implementation details.}
Our implementation of GeoCert \citep{jordan_provable_2019} starts with the original code provided by the authors at \url{https://github.com/revbucket/geometric-certificates} and modifies it to make use of open-source convex programming solvers to enable further research to a wider community, as the original code uses Gurobi \citep{gurobi}, which is a commercial software.  
We use CVXPY \citep{cvxpy} to model the convex programs and the open-source ECOS solver \citep{ecos} as the main solver.
Since ECOS can occasionally fail, we also use OSQP \citep{osqp} and SCS \citep{o2016conic} as backup solvers.
In our tests, ECOS and OSQP are significantly faster than SCS which we use only as a last resort.

We implemented LayerCert in Python using the packages numpy, PyTorch, numba (for the lower bounding methods), and the aforementioned packages for solving convex programs. 
Our experiments were performed on an Ubuntu 18.04 server on an Intel Xeon Gold 6136 CPU with 12 cores.
We restricted the algorithms to use only a single core and do not allow the use of the GPU. 
All settings, external packages, and compute are identical for all the methods compared. 

\subsection{Exact measurement experiments.}
We randomly picked 100 `1's and `7's from MNIST and collected the wall-clock time and number of linear programs solved. 
Since some instances can take a very long time to solve fully, we set a time limit of 1800s for each instance.

\paragraph{Averaged metrics.}
We measure (1) the average of the wall-clock time taken to measure the distances and (2) the average of number of linear programs solved. 
The first metric is what matters in practice but is heavily dependent on the machine and the solver used, whereas the second metric provides a system-independent proxy for time.
We present the results in Table~\ref{tb:results}. 
The basic GeoCert method is consistently outperformed in both metrics by our methods, while GeoCert-Lip always improves on GeoCert in terms of number of LPs and often in terms of run time. 
The method with the fastest run time is always one of the three LayerCert variants that use lower bounds. 
With the exception of the $3 \times [50]$ network where all methods timed out, the method with the least number of convex programs solved is always LayerCert-Both.
Plain LayerCert always outperforms both GeoCert methods in terms of timing, while it often outperforms GeoCert-Lip in number of LPs.
In some cases only one of the two lower bounding methods helps significantly, though the use of both methods at most includes a small overhead.
Thus, we recommend in general using both.

{

\paragraph{Performance profiles.}
Performance profiles \citep{dolan_benchmarking_2002} are a commonly-used technique to benchmark the performance of different optimization methods over \emph{sets of instances}. 
A performance profile is a plot of a cumulative distribution, where the $x$-axis denotes the amount of time each method is allowed to run for, and the $y$-axis the number of problems that can be solved within that time period. 
If the performance profile of a method is consistently above the performance profile of another method, it indicates that the former method is always able to solve more problems regardless of the time limit.
We illustrate the performance profiles for GeoCert-Lip and LayerCert-Both in Figure~\ref{fig:perfprof_inf}. LayerCert-Both consistently dominates GeoCert-Lip.



\section{Conclusion}

We have developed a novel hierarchical framework LayerCert for exact robustness verification that leverages the nested hyperplane arrangement structure of ReLU networks. 
We prove that a basic version of LayerCert is able to reduce the number and size of the convex programs over GeoCert using the equivalent priority functions. 
We showed that LayerCert is amenable to the use of lower bounding methods that use convex relaxations to both prune and warm start the algorithm. 
Our experiments showed that LayerCert can significantly outperform variants of GeoCert.

\section*{Acknowledgments}

We would like to thank the anonymous reviewers for suggestions that helped refine the presentation of this work.
We would also like to thank Eilyan Bitar, Ricson Cheng, and Jerry Liu for helpful discussions. 


\bibliography{verify}
\bibliographystyle{icml2020}

\clearpage

\onecolumn

\appendix


\section{An Example Comparing GeoCert and LayerCert-Basic}

The following simple example of a network with two hidden layers over a 2D input demonstrates the advantage of LayerCert-Basic over GeoCert. Recall the definition of the neural network from \eqref{eq:z}. We pick the following values for weights and biases of the network:
\begin{align*}
W_0 &\coloneqq \begin{bmatrix}
1 & 0 \\
0 & 1  
\end{bmatrix}, 
\quad b_0 \coloneqq \begin{bmatrix}
0   \\
0  
\end{bmatrix},
\\
W_1 &\coloneqq \begin{bmatrix}
1 & 1 
\end{bmatrix},
\quad\, b_1 \coloneqq -1,
\\
W_2 &\coloneqq \begin{bmatrix}
1 \\
0 
\end{bmatrix},
\qquad\;\, b_2 \coloneqq \begin{bmatrix}
0   \\
10  
\end{bmatrix}.
\end{align*}
Let the input point be $x \coloneqq [-1, -1.25]^\intercal$. The input space and the corresponding graphs for LayerCert and GeoCert are illustrated in Figure~\ref{fig:example}. We leave the decision boundary off the figure.
{
\vspace{3pt}
\begin{figure}[H]
\begin{center}
\includegraphics[width=0.5\linewidth]{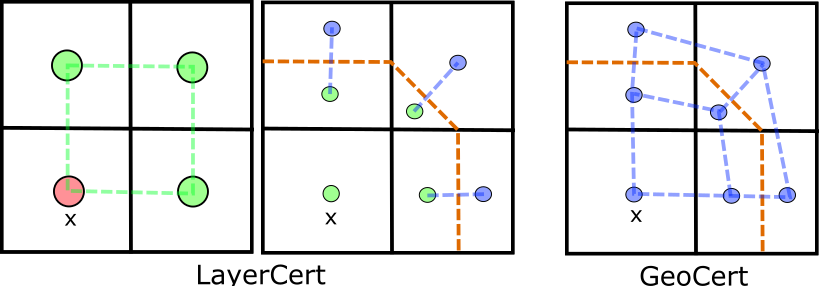}
\caption{Input space with hierarchical graph for LayerCert and neighborhood graph for GeoCert.
\label{fig:example}}
\end{center}
\end{figure}
\vspace{-14pt}
}
The next figure shows the order in which GeoCert processes the nodes and the edges (distance computations). At each node, we check if there is a decision boundary contained within the region. As for each each, we compute the distance from the input $x$ to the boundary of the two regions that is connected by the edge.
{
\vspace{3pt}
\begin{figure}[H]
\begin{center}
\includegraphics[width=0.35\linewidth]{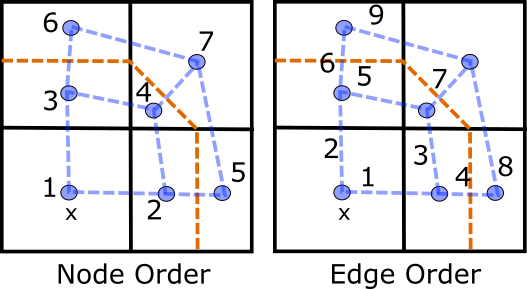}
\caption{GeoCert Node and Edge Processing Order.
\label{fig:example_geocert}}
\end{center}
\end{figure}
\vspace{-14pt}
}
For LayerCert, we visit more nodes. Note that we still visit all nodes in the last layer in the same relative order. 
We only need to check for decision boundaries for the nodes in the last layer. 
In addition, within the first layer, we do not have to compute the distance to the top-right green node more than once.
{
\vspace{3pt}
\begin{figure}[H]
\begin{center}
\includegraphics[width=0.35\linewidth]{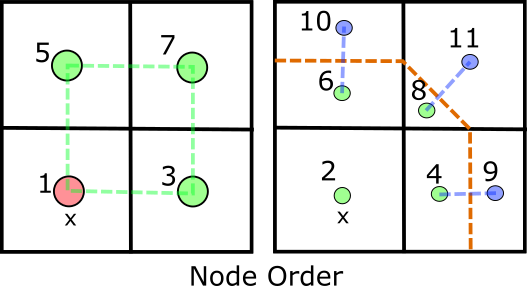}
\qquad
\includegraphics[width=0.35\linewidth]{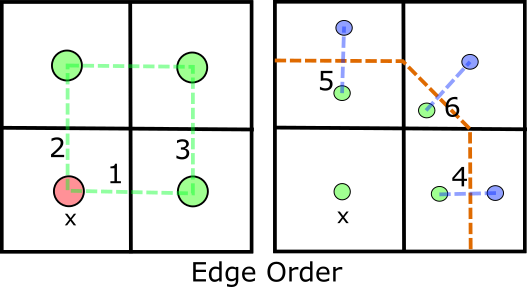}
\caption{LayerCert Node and Edge Processing Order.
\label{fig:example_layercert}}
\end{center}
\end{figure}
\vspace{-14pt}
}

\section{Proofs for LayerCert-Basic}

\subsection{Lemmas on Hyperplane Arrangements}

Before we prove the result for the full hierarchical structure, we focus on the simpler case of just hyperplane arrangements.

We show that for each pattern $A$, $\dist(R_A,x)$ can be obtained by considering \emph{any} neighboring pattern $A'$ such that $R_{A'}$ is closer to $x$ than $R_A$ is. 
This fact is allows us to skip the computation of some of the distances compared to GeoCert.  
\begin{proposition} \label{prop:hyperplane_arrangement_distance}
Let $\dist$ be a convex distance function.
Given a nonempty region $R_A$ and a nonempty neighboring region $R_{A'}$ such that $\dist(R_A,x) \geq \dist(R_{A'},x)$, we have $\dist(R_{A},x) = \dist(\Face(A,A'),x)$.
\end{proposition}

Proposition~\ref{prop:hyperplane_arrangement_distance} is a special case of the following lemma:
\begin{lemma}\label{lem:one_cut_eq}
Let $C$ be a closed convex set and $f,g: \bR^n \rightarrow \bR$ be convex functions. 
Suppose the sets $D \coloneqq C \cap \{ v \:|\: g(v) \leq 0 \}$ and $C\setminus D$ are nonempty and suppose $\min_{x\in {C \setminus D}} f(x) \leq \min_{x\in D} f(x)$. 
Then there exists $x^* \in \argmin_{x \in D} f(x)$ satisfying $g(x^*) = 0$. 
\end{lemma}

\begin{proof}
Pick some $z \in \argmin_{v \in C \setminus D} f(v)$ and $y \in \argmin_{v \in D} f(v)$. If $g(z) = 0$ or $g(y) = 0$, we are done. 
Suppose then that $g(z) > 0$ and $g(y) < 0$. 
Let $w = z - y$ and pick $\lambda \in (0,1)$ such that $g(y + \lambda w) = 0$. 
It follows that $y + \lambda w$ is in $D$ and by the convexity of $f$ we have $f(y + \lambda w) \leq f(y)$. 
Hence $y + \lambda w$ satisfies the claim.
\end{proof}

The next thing we show about hyperplane arrangements will be useful when we consider nested hyperplane structures.
\begin{proposition} \label{prop:increasing_sequence}
Let $C$ be a convex set in $\bR^n$ and consider a hyperplane arrangement. 
Let $\dist$ be a convex distance function.
Let $A$ be a pattern that contains a point in $\argmin_{v \in C} \dist(v,x)$ and consider another pattern $A'$ such that $R_{A'} \cap C$ is nonempty.
Then there is a sequence of patterns $A=A^1,A^2,\dotsc,A^{k},A^{k+1}=A'$ such that we have
\begin{itemize}
  \item $A^{i}$ and $A^{i+1}$ are neighboring patterns for $i \in [k]$,
  \item $\dist(R_{A^{i}}\cap C,x) \leq \dist(R_{A^{i+1}}\cap C,x)$ for $i \in [k]$, and
  \item $R_{A^i} \cap C$ is nonempty for $i \in [k+1]$.  
\end{itemize} 
\end{proposition}

Before we prove Proposition~\ref{prop:increasing_sequence}, we will prove the following result about the hyperplane arrangements. 
We use the terms `regions' and `neighbors' here similar to their use in Definitions~\ref{def:neighbors} and \ref{def:full_activation}.

\begin{lemma}\label{lem:line}
Consider a hyperplane arrangement, a convex set $C$, and a line with endpoints contained entirely within $C$. 
Consider the undirected graph $G$ where the nodes correspond the regions of the arrangement and edges are formed between neighboring regions. 
For all regions in the arrangement that have a nonempty intersection with the line, there is a path in $G$ between two regions that only passes through regions that touch the line.
\end{lemma}

\begin{proof}
We will first show that any two regions that share a common point can be connected through regions that share that same point $v$. 
Suppose this is true if there are $k$ hyperplanes going through the point $v$. Let $A$ and $A'$ be neighboring regions that contain $v$. 
After introducing another hyperplane through $v$, suppose $A$ gets split into two neighboring regions $A^1,A^2$. 
The addition of the hyperplane either splits $A'$ into two regions or keeps it as one. Let $A^3$ be one of the nonempty regions. 
It must be the case that $A^3$ is on the same side of the hyperplane as one of $A^1,A^2$, so $A^3$ is a neighbor to one of them. Hence, all the regions around the point remain connected to each other through each other.

We will now show we can move from one end of the line to the other. 
The hyperplane arrangement partitions the line into line segments $\{l_1,l_2,\dotsc,l_k\}$. 
There is a region that covers each line segment, so we can connect any two such line segment covering regions through a path. 
Furthermore, any region touching the line must touch one of the endpoints of the line segment, and hence there is a path between that point and the corresponding line segment covering region.
\end{proof}

\begin{proof}[Proof of Proposition~\ref{prop:increasing_sequence}]
Let $v$,$v'$ be the minimizers of $\dist(\cdot,x)$ when restricted to $R_{A} \cap C$ and $R_{A'} \cap C$ respectively. 
If $\dist(R_A\cap C,x) = \dist(R_{A'} \cap C)$, then the straight line segment connecting $v$ and $v'$ only passes through patterns with the same distance by convexity of the distance function and the fact that $A$ contains a minimizer of $\dist(\cdot, x)$ in $C$.

Suppose the set $C$ and corresponding pattern $A$ are fixed.
We now prove this for all valid $A'$ patterns by performing induction on the distance $\dist(R_{A'} \cap C, x)$ since there are only finitely many. 
Suppose this is true for all patterns of distance less than some $d$. 
Let $A'$ be a pattern where $R_{A'} \cap C$ has next lowest distance $d'$ from $x$.
Again consider the straight line segment connecting $v$ and $v'$.
All regions touching $v'$ have distance at most $d'$. 
Furthermore, since the number of regions is finite, by the convexity of the line segment there must be some pattern $B$ where $v' \in R_B$ and $R_B$ also contains other points on the line segment closer to $v$. The convexity of the distance function implies that $\dist(R_B \cap C,x) < d'$.
By Lemma~\ref{lem:line} there must be a series of neighboring regions connecting $B$ to $A'$, and as a result $A'$ is connected to a region of at most distance $d$ through regions of at most distance $d'$. 
\end{proof}

\subsection{Correctness Proof}

We return to considering partial activation patterns and Theorem~\ref{thm:correctness}, which claims that LayerCert-Basic returns the distance of the closest point with a different label from the input.
Recall the \emph{hierarchical search graph} described in Section~\ref{subsec:deep_geo}: 
\begin{itemize}
  \item Root node (level $0$): an empty `activation pattern' corresponding to the `activation region' $\bR^n$.
  \item Nodes at level $l$: the nonempty layer-$l$ partial activation region.
  \item Edges at level $l$: between $A$ and $A'$ if they are $l$-layer neighbors (i.e.\ $A$ and $A'$ differ on a single $l$th layer neuron).
  \item Edges between levels $l$ and $l+1$: between $\parent(A)$ and $A$ if $A$ is obtained by $\nextlayer(A,v)$ during the GeoCert algorithm.
\end{itemize}
In each iteration of the outer loop of LayerCert-Basic, the algorithm considers all nodes adjacent to those that have been processed. It picks a closest unprocessed node and computes the distance to all the neighboring nodes.

\begin{lemma}\label{lem:monotonic_paths}
Let $H$ denote the hierarchical search graph of LayerCert-Basic for some network and input point $x$. 
There is a path in $H$ from the initial layer-1 activation pattern $A^x$ to every pattern $A$ such that the distances of the patterns on the path are monotonically increasing.
\end{lemma}

\begin{proof}
Suppose the main claim holds for all patterns up to layer $l-1$. We will now show this for layer $l$. 
The claim holds for any pattern obtained through the $\nextlayer$ operation since the distance of this pattern is equal to that of its parent. 
Given an arbitrary $l$-layer region $R_A$ of distance $d$ from $x$, as a consequence of Proposition~\ref{prop:increasing_sequence} there is a path in $H$ of monotonically increasing distances from the pattern $B$ obtained by applying $\nextlayer$ to $\parent(A)$. 
\end{proof}

The proof of correctness follows from Lemma~\ref{lem:monotonic_paths}.

\begin{proof}[Proof of Theorem~\ref{thm:correctness}]
Since the distance to the distance boundary contained within a full region is further or equal to the distance of the region, if we process all the regions in order of distance we will terminate only when we have computed all regions closer than the closest decision boundary.

The first region processed is the region associated the empty activation pattern which has distance $0$ from the input point. 
Now suppose that we have popped all patterns of distance $< d$ from the priority queue and the next closest un-popped region is of distance $d$. 
We will now show that the priority queue contains a region of distance $d$ before the next iteration of the outer loop and that we have correctly computed the distance to this region.

Pick an unpopped activation pattern $A$ of distance $d$ of the shortest path length (in terms of number of edges in $H$) to the root node. 
If $A$ is directly connected to its parent $\text{parent}(A)$ in $H$, then we must have popped and processed $\text{parent}(A)$ in some iteration of the algorithm since that has a smaller path length to the root node in $H$.
This means that $A$ is in the priority queue. 
If this is not the case, then there is some sibling $B$ that is connected to $\text{parent(A)}$. 
By Lemma~\ref{lem:one_cut_eq}, we know that 
$$\dist(R_B,x) = \dist(R_{\text{parent}(A)},x) \leq \dist(R_A,x).$$ 
Proposition~\ref{prop:increasing_sequence} tells us that there is a path in $H$ between $B$ and $A$ of patterns with monotonically increasing distance. 
Pick the first unpopped pattern $A'$ on this path that is of distance $d$ from $x$. 
$A'$ must have a popped pattern next to it, and as a result $A'$ must be on the priority queue.
\end{proof}

\subsection{Main Proof}

\begin{proof}[Proof of Theorem~\ref{thm:main}]
From Theorem~\ref{thm:correctness}, we know that LayerCert-Basic processes the full activation patterns in order of increasing distance. GeoCert does the same, and therefore both methods will process the same full activation patterns. 

We will first show that the only nodes LayerCert processes are those that are full activation patterns or ancestors of the full activation patterns processed. 
Suppose for contradiction that this is not the case. Then, there is some processed pattern $A$ that is not a full activation pattern where we process no child. 
However, LayerCert will apply $\nextlayer$ to $A$, adding the child node $B$ with the same priority as $A$ to the priority queue. 
By assumption we have $U$ larger than the distance of $B$, so we will process $B$ eventually, contradicting our earlier claim.

For any full activation pattern $A$, there is a one-to-one mapping between the convex programs processed by GeoCert and the convex programs processed LayerCert for $A$ \emph{and all ancestors of} $A$. This is because we form one program for each neuron in the ReLU network. Consider the $j$-th neuron in the $l$-th layer. 
The constraints in the convex program formed by GeoCert for $A$ and a neighboring full pattern $A'$ are
\begin{align*} 
  A_i(j) z^i_j \geq 0 \text{ for } i \in [L], j \in [n_i], \text{ and } 
  z^{l}_{j'} = 0
\end{align*}
where $z^i_j$ is defined as in \eqref{eq:z}. In contrast, LayerCert uses 
\begin{align*} 
  A_i(j) z^i_j \geq 0 \text{ for } i \in [l], j \in [n_i], \text{ and } 
  z^{l}_{j'} = 0
\end{align*}
which only has a subset of the constraints. Since the only nodes we have to consider for LayerCert are ancestors of full activation patterns, this concludes the proof of the theorem.
\end{proof}

\section{Simplifying Projection Problems via Upper Bounds} \label{sec:reducing constraints}

To reduce the number of constraints in Problems \eqref{eq:decision} and Problems~\eqref{eq:priority}, \citet{jordan_provable_2019} note that we can quickly decide if a constraint of the form $a^\intercal x \leq b$ is necessary by checking if the hyperplane $\{x \:|\: a^\intercal x = b\}$ overlaps with the ball $B_{p,U}(x)$. 
This can be done via performing an $\ell_p$ projection of $x$ onto the hyperplane given by $a^\intercal x = b$, which can be done in a linear number of elementary operations.
This technique of using upper and lower bounds on variables (corresponding to the case where $p=\infty$ to prune constraints is standard in the linear programming literature (see for example \citet{gondzio_presolve_1997}).
Each time a tighter upper bound $U$ is obtained, the ball $B_{p,U}(x)$ shrinks, allowing us to remove more constraints.

We can apply this technique to GeoCert and any variant of the LayerCert method. Under the additional assumption that GeoCert and LayerCert-Basic process full activation patterns in the same order, LayerCert-Basic maintains its advantage in number and size of convex programs solved.

\begin{theorem}\label{thm:main_prune} (Complexity of LayerCert-Basic with Domain-based Constraint Pruning)
Suppose LayerCert-Basic and GeoCert process full activation patterns in the same order. 
Suppose we formulate the convex problems associated with $\decision$ and $\nextlayer$ using Formulations~\eqref{eq:region_poly} and \eqref{eq:rface_poly}. 
We can construct an injective mapping from the set of convex programs solved by LayerCert to the corresponding set in GeoCert such that the constraints in the LayerCert program is a subset of those in the corresponding GeoCert program.
\end{theorem}

The proof of this is similar to the proof of Theorem~\ref{thm:main}. 
The addition of domain-based pruning does not affect the result since the same constraints are pruned for both methods.

\section{LayerCert Framework and Convex Restrictions}

We expand on the discussion on the use in Section~\ref{sec:full} on how to use incomplete verifiers such as those presented in \citet{weng_towards_2018,zhang_efficient_2018} or the full linear programming relaxation of single ReLUs to improve LayerCert. 
Once we compute a convex set $M$ that contains the closest decision boundary, we can subsequently restrict our attention to just the intersection of activation regions with $M$. 
The correctness of this follows from the fact that Proposition~\ref{prop:increasing_sequence} accommodates the use of such a set $M$ and from modifying the hierarchical search graph in Lemma~\ref{lem:monotonic_paths}.

Instead of just applying $\restrict$ at the initiation iteration, one can apply it at every iteration. 
This version of $\restrict$ also takes in an activation pattern and forms the restriction based on this. 
We describe the full algorithm in Algorithm~\ref{alg:layercert-recursive}.
The use of $\restrict$ in every iteration can add significant overhead to the algorithm, so there is a trade-off between (1) the choice of $\restrict$ used and (2) the frequency in which $\restrict$ is applied, and these choices can depend heavily on the width and depth of the neural networks that we are verifying. 
We leave a computational study of this to future work. 

\begin{algorithm}
   \caption{LayerCert with Recursive Convex Restrictions}
   \label{alg:layercert-recursive}
\begin{algorithmic}[1]
  \STATE {\bfseries Input:} $x$, $y$ (label of $x$), $U$ (upper bound)
  \STATE $d,A,M \gets \restrict(\emptyset,x,ub)$
  \STATE $Q \gets $ empty priority queue
  \STATE $Q.\push((d,A,M))$
  \STATE $S \gets \emptyset$
  \WHILE{$Q \neq \emptyset$}
    \STATE $(d, A, M) \gets Q.\pop()$
    \IF{$ub \leq d$}
      \STATE \textbf{return} $U$
    \ENDIF
    \IF{$A$ is a full activation pattern}
      \STATE $ub \gets \min(\decision(A, x),ub)$
    \ELSE
      \IF{$\reachable(A,x,ub) =$ `maybe'}
          \STATE $d',A',M' \gets \restrict(A,x,ub)$
          \STATE $Q.\push(d',A',M')$
      \ENDIF
    \ENDIF
    \FOR{$A' \in N_{\text{current\_layer}}(A)\setminus S$}
      \IF{$\Face(A,A') \cap M$ is nonempty}
        \STATE $d' \gets \priority(x, \Face(A,A') \cap M)$ 
        \STATE $Q.\push((d',A',M))$
        \STATE $S \gets S \cup \{ A' \}$
      \ENDIF
    \ENDFOR
  \ENDWHILE
\end{algorithmic}
\end{algorithm}

\section{Additional Experiments and Details}


\subsection{Parameter Choices}

\paragraph{Convex programming software. }
We used the following package versions and settings for handling convex programs. All the solvers are accessed through CVXPY's interface:
\begin{itemize}
  \item Modeling language: CVXPY v1.0.25 \citep{cvxpy}.
  \item Primary solver: ECOS v2.0.7  \citep{ecos} --
    `abstol':$1 \times 10^-5$,
    `reltol':$1 \times 10^-4$,
    `feastol':$1 \times 10^-5$,
    `abstol\_inacc':$5 \times 10^-3$,
    `reltol\_inacc':$5 \times 10^-2$,
    `feastol\_inacc':$5 \times 10^-3$.
  \item Secondary solver: OSQP v0.6.0 \citep{osqp} --
    `eps\_abs':$0.001$, 
    `eps\_rel':$0.001$.
  \item Backup solver: SCS v2.1.1 \citep{o2016conic} --
    default settings.
\end{itemize}

\paragraph{Use of constraint-reducing techniques.}

We used the same domain-based techniques mentioned in Section~\ref{sec:reducing constraints} for all the methods.
Before adding a constraint to a convex program to solve, we check if the constraint overlaps with $B_{p,U}(x)$.

\subsection{Additional Results}

\paragraph{$\ell_2$-norm experiments.}

In Table~\ref{tb:results_l2} we present the averaged results for $\ell_2$-norm. We used an upper bound radius of 3.0. As with the $\ell_\infty$ results presented in Table~\ref{tb:results}, the LayerCert methods consistently outperform the GeoCert methods in terms of the timing. We note that some of the difference in running times is partially due to the solvers -- ECOS and OSQP failed frequently, and whenever SCS is invoked the running time is significantly increased. This issue persisted over a range of hyperparameter settings for ECOS and OSQP. We note that this is rare for the $\ell_\infty$-norm experiments.

As for the number of quadratic programs, LayerCert-CROWN and LayerCert-Both consistently outperform both GeoCert methods, while LayerCert and LayerCert-IA mostly but not always outperforms GeoCert-Lip. GeoCert is consistently the worst performing method in terms of number of QPs and running time.

{
\begin{table*}
\begin{center}
\includegraphics[width=1.0\linewidth]{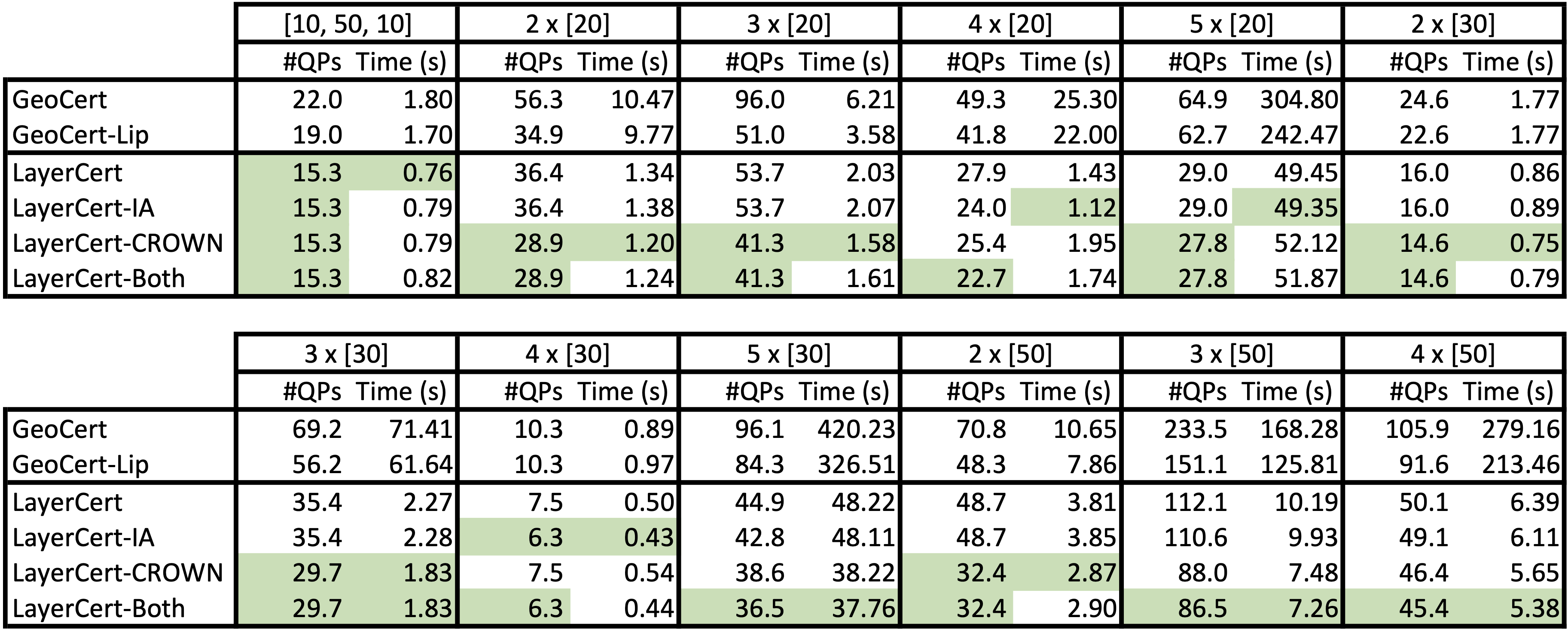}
\vspace{-5pt}
\caption{
Average number of convex programs and running time over 100 inputs for 12 different networks for $\ell_2$-norm distance.
The green shaded entries indicate the best performing method for each neural network under each metric.
The gray shaded entries indicates the algorithm timed out before an exact solution to Problem~(\ref{eq:robustness}) could be found or before a radius of 3 is reached for \emph{at least one input}.
Whenever a timeout occurs, we use a time of 1800 seconds in its place, which leads to an underestimate of the true time.
\label{tb:results_l2}}
\end{center}
\end{table*}
}

\paragraph{Ablation Study for LayerCert-Basic.}

In Section~\ref{subsec:layercert-basic}, one of the reasons provided for the reduction in the number of $\priority$ computations for LayerCert-Basic over GeoCert was that was LayerCert marked a pattern $A$ as `seen' after the first $\priority$ computation involving $A$, whereas GeoCert only did so after it selected $A$ at the start of an iteration. To test the effect of this particular change, we created a new variant of LayerCert where we only marked a pattern as seen after it is chosen at the start of a main iteration, akin to GeoCert. 

The results are presented in Table~\ref{tb:ablation}. We see that this leads to a slight overhead over LayerCert, but overall this accounts only for a small fraction of the improvement in performance over GeoCert. This indicates that the majority of the improvement comes from hierarchical structure and the way that it allows us to avoid a large number of $\priority$ computations while reducing the complexity of each computation.

{
\begin{table*}
\begin{center}
\includegraphics[width=1.0\linewidth]{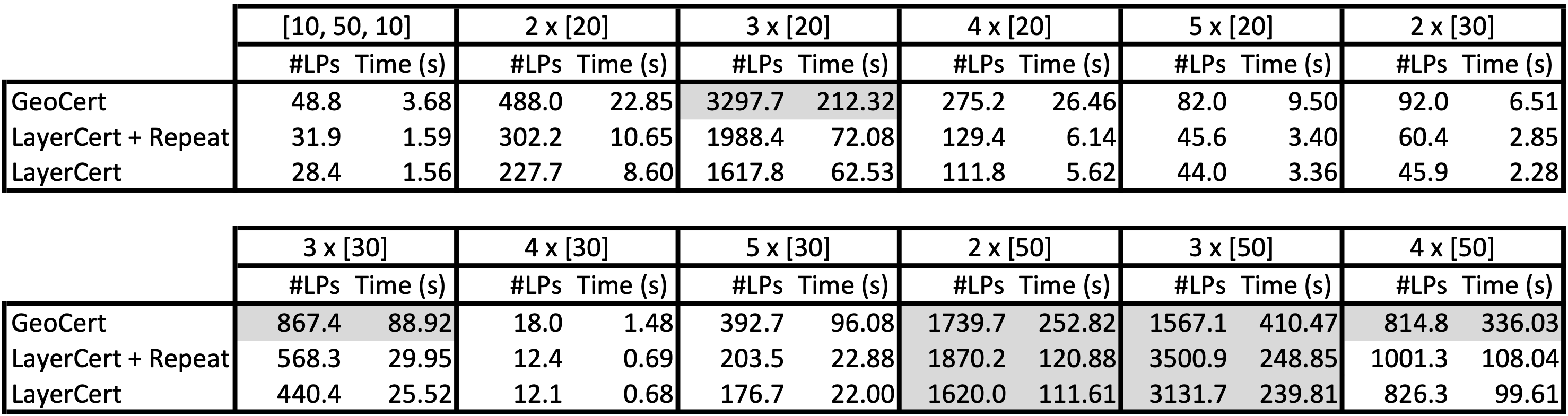}
\vspace{-5pt}
\caption{Ablation Study:
Average number of convex programs and running time over 100 inputs for 12 different networks for $\ell_\infty$-norm distance, focusing on GeoCert, LayerCert-Basic, and a variant of LayerCert that allows repeated $\priority$ computations by shifting when we mark a pattern as `seen'.
The gray shaded entries indicates the algorithm timed out before an exact solution to Problem~(\ref{eq:robustness}) could be found or before a radius of 0.3 is reached for \emph{at least one input}.
Whenever a timeout occurs, we use a time of 1800 seconds in its place, which leads to an underestimate of the true time.
\label{tb:ablation}}
\end{center}
\end{table*}
}

\paragraph{Experiments for Larger Networks}

To compare the performance of LayerCert and GeoCert for large networks, we chose the following fully-connected networks and used $\ell_\infty$-distance.
\begin{itemize}
\item $8\times[20]$,
\item $6\times[30]$,
\item $5\times[50]$,
\item $5\times[60]$.
\end{itemize}
We chose the 10 instances where none of the methods were able to converge within 1800 seconds and compared the performance of GeoCert with Lipschitz term and LayerCert with both heuristics. Figures \ref{fig:il-820}, \ref{fig:il-630}, \ref{fig:il-550}, \ref{fig:il-560} contain these results. The orange solid line in each represent LayerCert while the blue dotted line represents GeoCert.

{
\begin{figure}
\begin{center}
\includegraphics[width=0.16\linewidth]{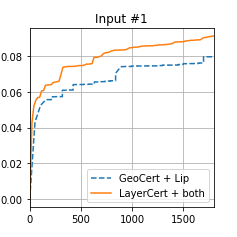}
\includegraphics[width=0.16\linewidth]{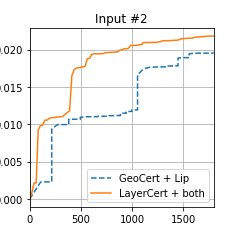}
\includegraphics[width=0.16\linewidth]{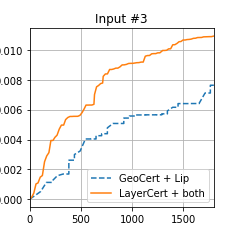}
\includegraphics[width=0.16\linewidth]{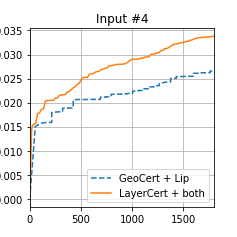}
\includegraphics[width=0.16\linewidth]{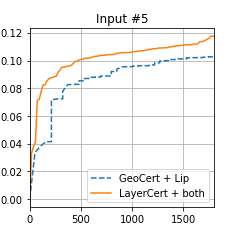}
\\
\vspace{-5pt}
\includegraphics[width=0.16\linewidth]{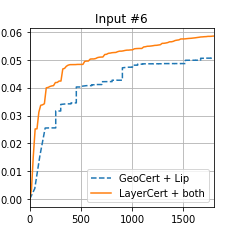}
\includegraphics[width=0.16\linewidth]{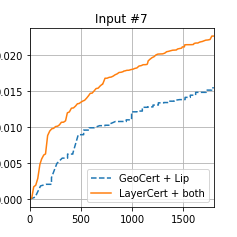}
\includegraphics[width=0.16\linewidth]{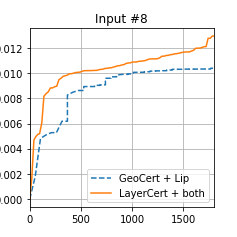}
\includegraphics[width=0.16\linewidth]{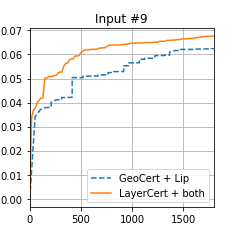}
\includegraphics[width=0.16\linewidth]{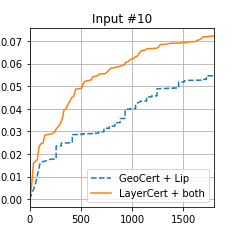}
\\
\vspace{-5pt}
\caption{Timing for 10 Instances, $8\times[20]$ network. \label{fig:il-820}}
\end{center}
\end{figure}
}

{
\begin{figure}
\begin{center}
\includegraphics[width=0.16\linewidth]{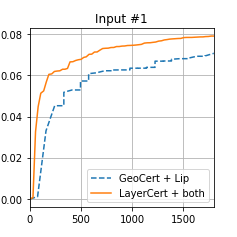}
\includegraphics[width=0.16\linewidth]{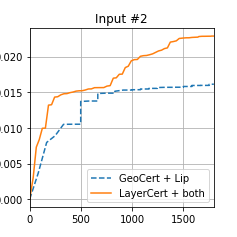}
\includegraphics[width=0.16\linewidth]{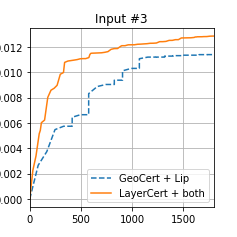}
\includegraphics[width=0.16\linewidth]{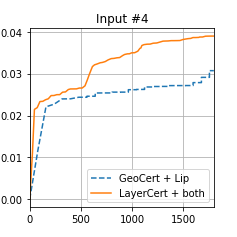}
\includegraphics[width=0.16\linewidth]{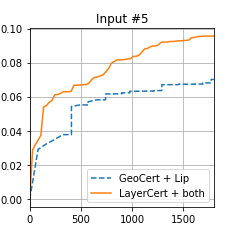}
\\
\vspace{-5pt}
\includegraphics[width=0.16\linewidth]{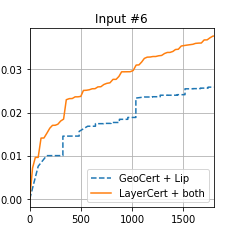}
\includegraphics[width=0.16\linewidth]{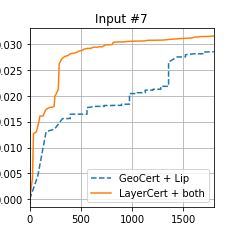}
\includegraphics[width=0.16\linewidth]{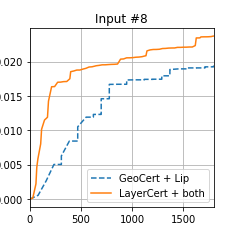}
\includegraphics[width=0.16\linewidth]{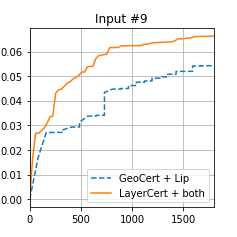}
\includegraphics[width=0.16\linewidth]{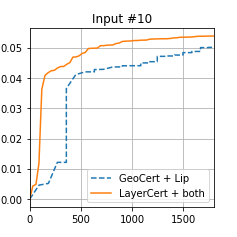}
\\
\vspace{-5pt}
\caption{Timing for 10 Instances, $6\times[30]$ network. \label{fig:il-630}}
\end{center}
\end{figure}
}
{
\begin{figure}
\begin{center}
\includegraphics[width=0.16\linewidth]{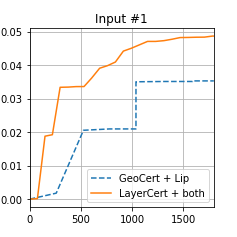}
\includegraphics[width=0.16\linewidth]{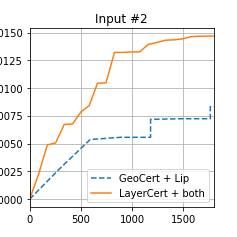}
\includegraphics[width=0.16\linewidth]{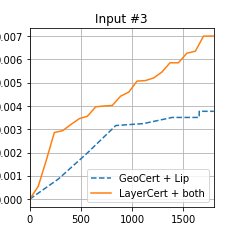}
\includegraphics[width=0.16\linewidth]{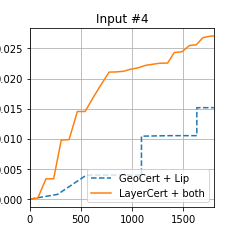}
\includegraphics[width=0.16\linewidth]{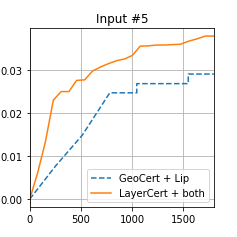}
\\
\vspace{-5pt}
\includegraphics[width=0.16\linewidth]{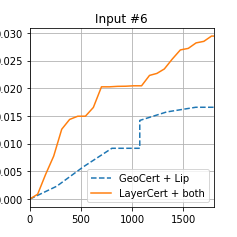}
\includegraphics[width=0.16\linewidth]{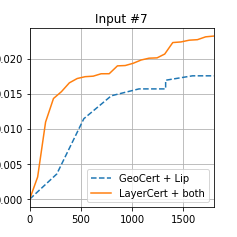}
\includegraphics[width=0.16\linewidth]{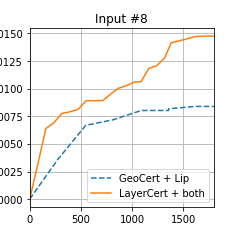}
\includegraphics[width=0.16\linewidth]{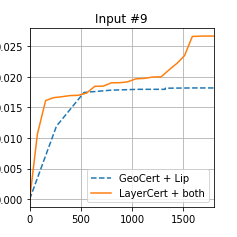}
\includegraphics[width=0.16\linewidth]{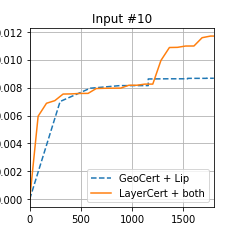}
\\
\vspace{-5pt}
\caption{Timing for 10 Instances, $5\times[50]$ network. \label{fig:il-550}}
\end{center}
\end{figure}
}
{
\begin{figure}
\begin{center}
\includegraphics[width=0.16\linewidth]{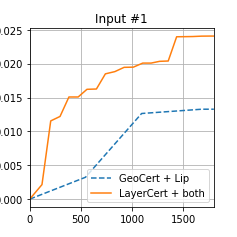}
\includegraphics[width=0.16\linewidth]{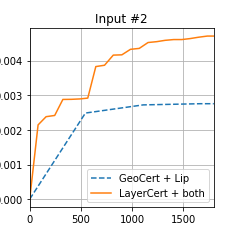}
\includegraphics[width=0.16\linewidth]{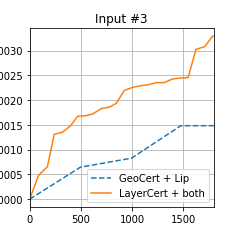}
\includegraphics[width=0.16\linewidth]{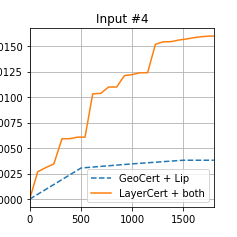}
\includegraphics[width=0.16\linewidth]{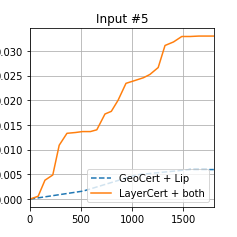}
\\
\vspace{-5pt}
\includegraphics[width=0.16\linewidth]{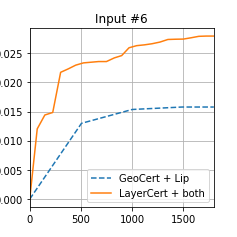}
\includegraphics[width=0.16\linewidth]{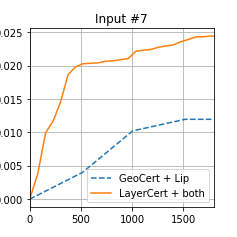}
\includegraphics[width=0.16\linewidth]{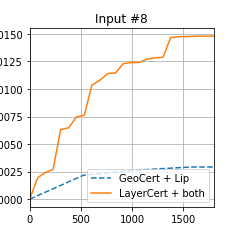}
\includegraphics[width=0.16\linewidth]{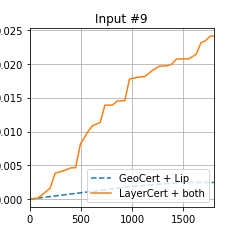}
\includegraphics[width=0.16\linewidth]{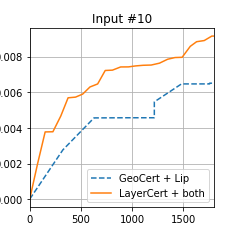}
\\
\vspace{-5pt}
\caption{Timing for 10 Instances, $5\times[60]$ network. \label{fig:il-560}}
\end{center}
\end{figure}
}

\end{document}